\documentclass{article}
\usepackage[utf8]{inputenc}

\newcommand{\Bem}[1]{}
\usepackage{amssymb}
\usepackage{amstext}
\usepackage{amsmath}
\usepackage{amsthm}
\usepackage{latexsym}

\newcommand{\keywords}[1]{\textbf{\textit{Keywords: }} #1}

\usepackage{url}
\usepackage{amssymb,amsmath}
\usepackage{ifpdf}
\ifpdf
\usepackage[colorlinks,linkcolor=blue,urlcolor=blue,citecolor=blue,
plainpages=false,pdfpagelabels,breaklinks]{hyperref}
\else
\usepackage[colorlinks,linkcolor=blue,urlcolor=blue,citecolor=blue,
plainpages=false,pdfpagelabels,linktocpage]{hyperref}
\fi
\usepackage{hyperref}
\title{High-Dimensional Wide Gap $k$-Means Versus Clustering Axioms}
\author{Mieczys{\l}aw A. K{\l}opotek }
\date{December 2022}

\begin{document}

\newcommand{\eref}[1]{(\ref{#1})}
\newtheorem{twierdzenie}{Theorem}
\newtheorem{definicja}{Definition}
\newcommand{\fhl}{}
\newcommand{\hl}{}
\newcommand{\st}{}
\newcommand{\T}{^T}
\newcommand{\K}{K}

\newcommand{\FI}{\left(\mathbf{I}-\mathbf{1}\mathbf{s}^T\right)}
\newcommand{\Fi}[1]{\left(\mathbf{I}-\mathbf{1}\mathbf{#1}^T\right)}
\newcommand{\FIT}{\left(\mathbf{I}-\mathbf{s}\mathbf{1}^T\right)}
\newcommand{\FiT}[1]{\left(\mathbf{I}-\mathbf{#1}\mathbf{1}^T\right)}
\newcommand{\FIOne}{\left(\mathbf{I}-\frac{\mathbf{1}\mathbf{1}^T}{m}\right)}

\newcommand{\CM}{\left(\mathbf{I}-\frac{\mathbf{1}\mathbf{1}^T}{m}\right)} 
\newcommand{\nonI}{\left( {\mathbf{1}\mathbf{1}^T} -\mathbf{I}\right)} 
\newcommand{\OO}{\mathbf{1}\mathbf{1}^T} 

\newcommand{\KMEANSSTUFF}[1]{#1}
\newcommand{\KS}[1]{#1}

\maketitle
\begin{abstract}
Kleinberg's axioms for distance based clustering proved to be contradictory. 
Various efforts have been made to overcome this problem. 
Here we make an attempt to handle the issue by embedding in high-dimensional space and granting wide gaps between clusters. 
\end{abstract}

\section{Introduction} 
Kleinberg \cite{Kleinberg:2002} proposed three seemingly obvious clustering axioms: 
richness, scale-invariance and consistency. 
The \emph{consistency axiom} states in particular, that a clustering function $f(d)$ clustering some set $S$ (with distance $d$ between datapoints) into the set of clusters $\Gamma$, st return the same clustering $\Gamma$ if applied to a different distance function $d'$ such that $d(P,Q)\ge d'(P,Q)$ if $P,Q$ are from different clusters of $\Gamma$ and  $d(P,Q)\le d'(P,Q)$  if $P,Q$ are from the same cluster.
They proved to be  contradictory. Is therefore clustering unnatural? 
There are several unnatural components of Kleinberg's axiomatic system. 
First of all, we should expect a proper behaviour of a clustering algorithm only if the data as such is {cluster}able. $k$-means algorithm would cluster anything, the question is whether or not the outcome makes any sense. 

Therefore we consider {data}sets where the clustering is there. We shall speak about $k$-clustering if we cluster into $k$ groups. 

We will refer in this note predominantly to the widely used $k$-clustering algorithm called $k$-means. 
$k$-means algorithm was designed  to operate primarily in the Euclidean space. 
Let us recall the $k$-means quality function which is minimized by $k$-means.

 $$Q(\Gamma)=\sum_{i=1}^k \sum_{\mathbf{x}\in C_i} ||\mathbf{x}-\boldsymbol\mu(C_i)||^2$$. 
 which may be reformulated as 

\begin{equation} \label{eq:Q::kmeans}
Q(\Gamma)=\sum_{i=1}^m\sum_{j=1}^k u_{ij}\|\mathbf{x}_i - \boldsymbol{\mu}_j\|^2
=\sum_{j=1}^k \frac{1}{2n_j} \sum_{\mathbf{x}_i \in C_j} 
\sum_{\mathbf{x}_l \in C_j} \|\mathbf{x}_i - \mathbf{x}_l\|^2  
\end{equation} 

Let us recall that according to Kleinberg, $k$-means \emph{is not a clustering algorithm} as it fails on richness and consistency axioms. 

Our basic contribution is to show that we can create a version of $k$-means algorithm that meets the Kleinberg axiom set intuition for {data}sets that have well defined clusterability property. 

The note is structured as follows:
We recall earlier work in Section \ref{sec:prevWork}.
Then we propose a first approach to the solution in Section \ref{sec:varsep}. 
While it works under Kleinberg's definition of "distance", it is not quite suitable for $k$-means algorithm because it leads outside of the Euclidean space.
Therefore we modify our approach in Section \ref{sec:residualsep}.
Section \ref{sec:discussion} is devoted to a brief discussion of our outcome and Section \ref{sec:conclusions} summarizes the results and outlines further research directions.

\section{Previous Work} \label{sec:prevWork}

Obviously  it is not acceptable and is disastrous  for the domain of clustering algorithms if an axiomatic system consisting of ''natural axioms'' is self-contradictory. It means in plain words that the domain of clustering algorithms is a kind of fake science. 
Therefore numerous efforts have been made to cure such a situation by proposing different axiom sets or modifying Kleinberg's one mention Kleinberg himself introduced the concept of partition $\Gamma'$ being a refinement of a partition $\Gamma$, 
if for every set $C' \in \Gamma'$, there is a
set $C \in  \Gamma$ such that $C' \subseteq C$.
He  defines Refinement-Consistency, a relaxation of Consistency, to require that
if distance d' is an f (d)-transformation of d, then f(d') should be a refinement of f(d).
Though  there is no clustering function that satisfies Scale-{In}variance, Richness, and Refinement-Consistency, but if one defines Near-Richness as Richness without  the partition in which each element is in a separate cluster, then  there exist clustering functions f that satisfy Scale-{In}variance and
Refinement-Consistency, and Near-Richness (e.g.  single-linkage with the distance-($\alpha\delta$) stopping condition, where
$\delta=  min_{i,j} d(i, j)$ and $\alpha\ge 1$.  

The refinement consistency is not a good idea as the idea what is a cluter and what not is not appropriately reflected. Furthermore, it does not repair the $k$-means classification as a non-clustering algorithm.   
To overcome Kleinberg's contradictions, 
\cite{Ben-David:2009} proposed to axiomatize clustering quality function and not the clustering function itself. This is a bad idea because no requirements are imposed onto the clustering function. 
\cite{vanLaarhoven:2014} proposes to go over to the realm of graphs and develops a set of axioms for graphs. The approach is not applicable to $k$-means. 
\cite{%
Ackerman:2010NIPS} and \cite{Meila:2005} proposed to use the "axioms" not as a requirement to be met by all algorithms, but rather as a way to classify clustering functions.  
\cite{Strazzeri:2021} suggests to change the consistency axiom for graphs. 
\cite{Hopcroft:2012}  proposes to seek only clusters with special properties, in his case  to cluster the {data}sets into equal size clusters.

We have also proposed several approaches to removing the contradictions in the Kleinberg's axiomatic system, see e.g. \cite{MAKRAK:2020:limcons,MAKSTWRAK:2020:motion,MAKRAK:2020:fixdimcons,RAKMAK:2019:probrich,MAK:2022:kmeanspreserving,RAK:MAK:2019:perfectball,Klopotek:2022continuous,ISMIS:2022:richness}. All of them were based by assumption that a cluster should be enclosed in a kind of hyper-ball and the distances between the {hyper}balls shall be kept large enough. In this note we take a different stand referring rather to variation within a cluster. 

We have demonstrated that the consistency axiom alone is self-contradictory and proposed a remedy in terms of a convergent consistency \cite[Section 3]{MAKRAK:2020:limcons}.
Kleinberg's consistency transformation allowed for creation of new clusters within a cluster which led directly to the contradiction of the axiomatic system. 
Convergent consistency makes creation of clusters within a cluster impossible because it requires that when decreasing distances within a cluster, the shorter distances are reduced by lower percentage than longer ones.  

It has to be noted that our and other efforts in removing Kleinberg's axiomatic contradictions went along the research on so-called clusterability.  As mentioned, \cite{Hopcroft:2012} made a suggestion that restricting oneself to special data structures can overcome Kleinberg's contradictions. That is one looks rather at clustering of data that fulfil some properties of clusterability. 
Though a number of attempts have been made to capture formally the intuition behind clusterability, none of these efforts seems to have been successful, as Ben-David exhibits in \cite{Ben-David:2015} in depth.  A   paper by Ackerman et al.  \cite{Ackerman:2016} partially eliminates some of these problems, but regrettably at the expense of non-intuitive  user-defined parameters. 
As Ben-David mentioned,   the research in the area  does not address popular algorithms  except for $\epsilon$-{Separated}ness clusterability criterion related to $k$-means proposed  by Ostrovsky et al. \cite{Ostrovsky:2013}.
We have made some efforts in this direction in \cite{MAK:2017:clusterability}. This note also refers to clusterability while clustering via $k$-means.

\section{Variation{al} Cluster Separation} \label{sec:varsep}

Let us introduce a couple of useful concepts. 
First of all recall the fact that Kleinberg's consistency axiom leads definitely outside of the domain of Euclidean space. Therefore, to work with $k$-means algorithm, we need a reformulation of the $k$-means cluster quality function.
Let us first recall Kleinberg's "distance" concept. 
\begin{definicja}
For a given discrete set of points $S$, 
the function $d: S \times S \rightarrow \mathbb{R}$ will be called  a pseudo-distance function  iff  $d(x,x)=0$, $d(x,y) =d(y,x)$ and   $d(x,y) > 0$ for distinct $x,y$.
\end{definicja}

Following the spirit of kernel $k$-means as was exposed in \cite{RAKMAKSTW:2020:trick}, 
let us reformulate now the $k$-means cluster quality function in terms of this pseudo-distance. 

Define the function $Q(\Gamma,d)$  as follows:
\begin{equation}\label{eq:kmeansPseudoDist}
Q(\Gamma,d)    
=\sum_{j=1}^k \frac{1}{2n_j} \sum_{ i \in C_j} 
\sum_{ l \in C_j} d(i,l) ^2  
\end{equation}   
where $\Gamma$ is a clustering (split into disjoint non-empty subsets of cardinality at least 2) of a dataset $S$ into $k$ clusters, $n=|S|, n_j=|C_j$, and $d$ is a pseudo-distance function defined over $S$. 

Let us introduce our concept of well-{separated}ness. 
\begin{definicja} \label{def:varsep}
Let us consider a set of clusters separated as follows: 
Lat $\Gamma=\{C_1,\dots,C_k\}$ be a partition of the dataset $S$, $d$ be a pseudo-distance. 
Let   
\begin{equation}\label{eq:vardist}
d(i,l)>\sqrt{2}\sqrt{Q(\Gamma,d)} 
\end{equation}
for each $i,l$ 
such that $i$ belongs to a different cluster than $l$ under $\Gamma$.  
Then we say that the set $S$ with distance $d$ is 
\emph{ variation{ally} $k$-separable}. 
\end{definicja}

It is easily seen that in such a case
\begin{twierdzenie}
If the pseudo-distance $d$ fulfills the condition (\ref{eq:vardist}) 
under the clustering 
$\Gamma$ of $S$, then $\Gamma$  is the optimal $k$-clustering of   $S$ with $d$ under  kernel $k$-means.     
\end{twierdzenie}
\begin{proof}
Assume to the contrary that not $\Gamma$ but $\Gamma'$ different from it is the optimal $k$-clustering of $S$. 
$\Gamma'$ would then contain at least one cluster $C'$ with at least two {data}points $P,Q$ such that both stem from distinct clusters of $\Gamma$.
Hence their distance amounts to at least $\sqrt{2}\sqrt{Q(\Gamma)}$. 
All the other $n'-2$ elements of $C'$ do not belong under $\Gamma$ to the same cluster as $P$ or $R$. Let these {data}points be $n_P$ and $n_R$. Clearly, $n_P+n_R\ge n'-2$. So, within the cluster $C'$ there are at least $(n_P+1)*(n_R+1)$ pairs of {data}points with distance  at least  $\sqrt{2}\sqrt{Q(\Gamma)}$. 
So the contribution of $C'$ to the quality function amounts to 

\begin{equation} 
Q(\{C'\} ,d)    
=\frac{1}{2n'} \sum_{ i \in C'} 
\sum_{ l \in C'} d(i,l) ^2  
\end{equation}   
$$\ge \frac{1}{n'}  
(n_P+1)*(n_R+1) *2 Q(\Gamma,d)
\ge \frac{1}{n'}  
(n'-1) *2 Q(\Gamma,d)
$$
As $ Q(\Gamma',d)\ge Q(\{C'\} ,d) $,
so we have $  Q(\Gamma',d)\ge  Q(\Gamma,d)\ $ as claimed in this theorem. 
   $\Gamma$ is in fact optimal.     
\end{proof}

As the theorem holds for pseudo-distance, it holds also for the Euclidean distance. 
Let us ask the question how difficult it would be to discover the optimal clustering. 
Let us consider the $k$-means$++$ algorithm \cite{CLU:AV07}, or more precisely the derivation of the initial clustering. Recall that wide gaps between clusters guarantee that after hitting each cluster during the initialization stage, the optimum clustering is achieved. 
Let us consider a step when $i$ seeds have hit $i$ distinct clusters. 
Then the probability of hitting an unhit cluster in the next step amounts to:
$$\frac{SSD_{unhit}}{SSD_{unhit}+SSD_{hit}}=1- \frac{SSD_{hit}}{SSD_{unhit}+SSD_{hit}}$$
where $SD_{unhit}$ is the sum of squared distances to closest seed from elements of unhit clusters, and   $SSD_{hit}$ is the sum of squared distances to closest seed from elements of hit clusters. Obviously 
$$E(SSD_{hit})= 2 Q(\Gamma,d)$$
as the expected squared distance from cluster center amounts to $Q(\{C\},d$). On the other hand
$$SSD_{unhit} \ge 2 Q(\Gamma,d) \sum_{j=i+1}^k n_j $$
If we assume that  the cardinality of all clusters is the same and equals $m$, then we have 
$$ 
\frac{SSD_{unhit}}{SSD_{unhit}+SSD_{hit}} \ge
1-\frac{2 Q(\Gamma,d)}{2 Q(\Gamma,d)*(k-i)m+2 Q(\Gamma,d)}
=1- \frac{1}{m(k-i)+1}
$$
So that the overall expected probability of hitting all clusters during initialization amounts to at least 
\begin{equation}\label{eq:hitprob}
    \prod_{i=1}^{k-1} \left(1- \frac{1}{m(k-i)+1} \right)
\end{equation}

If $m$ exceeds $k$, then this probability is very close to one (given $m>50$). If not all clusters are of the same cardinality, but $m$ is its lower bound, then the above formula gives the lower bound on this probability.  

Now we are ready to formulate our well-separated clustering concept and the algorithm to detect it. 

\begin{definicja}
We say that a clustering function $f(d)$ returns
\emph{ variational $k$-clustering} of $S$ 
if $S$  variationally $k$-separable under $d$ and $f(d)$ returns the $\Gamma$ clustering used in the definition of  
\emph{ variational $k$-separability } - Def. \ref{def:varsep} . 
\end{definicja}

\begin{definicja}
We say that a clustering function $f(d)$ returns
\emph{ variational range-$k_x$ clustering} of $S$ 
if $S$  variationally $k$-separable under $d$ for some $1\le k\le k_x $ and $f(d)$ returns the $\Gamma$ clustering used in the definition of  
\emph{ variational $k$-separability } - Def. \ref{def:varsep} 
and
for no cluster $C\in \Gamma$
there exists  $k'$, $2\le k`\le k_x-k+1$ that  $C$ is 
variationally $k'$-separable. 
The maximal $k$ with this property shall be called the level of variational range-$k_x$ clustering.
\end{definicja}

Obviously, $k$-means++ would be a suitable sub-algorithm for this task of the following master algorithm: try out all k=2 to $k_x$ if there exist variational $k$-clustering; and if so, then check each sub-cluster on no variational $k'$ separability.

Now consider the essential of our considerations: what will happen when performing Kleinberg's consistency operation. 

The first problem that we encounter is that consistency transform performed on acluster tht is not variationally $k'$ separable, may turn to variationally $k'$ separable one. 
Therefore we need to restrict consistency transformation to a discrete convergent one. 
It has been proven in \cite[Section 3 on convergent consistency]{MAKRAK:2020:limcons} that in continuous case the consistency transform hat to reduce to a scaling of all distances by the same factor. 
However, if we consider a discrete set of points, then we may be less rigid:
\begin{definicja}
    A discrete convergent consistency transformation is Kleinberg's consistency transformation with the limitation that within one cluster shorter distances are shortened by fewer percent than longer ones. 
\end{definicja}

\begin{twierdzenie}
    The variational range-$k_x$ clustering at the level $k$ will remain the variational range-$k_x$ clustering at the level $k$ after discrete convergent consistency transform. In other words  discrete convergent consistency transform preserves clustering by a function detecting variational range-$k_x$ clustering. 
\end{twierdzenie}
\begin{proof}
The increase of inter-cluster distances does not violate variational $k$-separation. 
The decrease of intra-cluster separation does not turn a non-variationally separable set into a separable set. 
\end{proof}

Let us recall, that Kleinberg's richness property of the clustering function $f$ means that for any given set $S$ and any split $\Gamma$ of it, one can construct a distance function $d$ so that $f(d)$ returns the clustering $\Gamma$. 
We shall restrict this definitin to 
\begin{definicja}
The clustering function $f$ has  \emph{range-$k_x$ richness} property, if for any given set $S$ and any split $\Gamma$ of $S$ into non-empty subsets of at least two elements, one can construct a distance function $d$ so that $f(d)$ returns the clustering $\Gamma$. 
\end{definicja}

This implies 
\begin{twierdzenie}
     A function detecting variational range-$k_x$ clustering has the property of scale-invariance, discrete convergent consistency and range-$k_x$ richness. 
\end{twierdzenie}

However, with the  discrete convergent consistency transformation of a distance $d$ to a distance $d'$ we will have the problem, that even if $d$ is an Euclidean distance, $d'$ does not need to be an Euclidean distance. 
As shown in \cite{RAKMAKSTW:2020:trick}, a distance function $d'$ being non-euclidean can be turned into Euclidean one $d"$ by adding an appropriate constant $\delta^2$ to each squared distance $d'(i,j)^2$ of distinct elements and the clustering with (kernel) $k$-means will preserve the $k$-clustering of $S$. However, it is possible that the property of variational $k$ separability will be lost. Our goal, yet, is to find the class of {data}sets and clustering functions fitting axioms that operate in the Euclidean space.

\section{Residual Cluster Separation} \label{sec:residualsep}

Assume that $\sigma(d)$ is the lowest distance $d$ over the set $S$. 
Then we have

\begin{twierdzenie}
    $$Q(\Gamma,d) \ge (n-k) \frac{\sigma(d) ^2}{2}  $$
\end{twierdzenie}
\begin{proof}   
$$
Q(\Gamma,d)
\ge 
\sum_{j=1}^k \frac{1}{2n_j} \sum_{\mathbf{x}_i \in C_j} 
\sum_{\mathbf{x}_l \in C_j} \sigma(d) ^2  
= \sum_{j=1}^k \frac{n_j-1}{2} \sigma(d) ^2  
= (n-k) \frac{\sigma(d) ^2}{2} 
$$
\end{proof}
Define the function 
$$ \beta(\Gamma,d)=2\left(Q(\Gamma,d)-(n-k-1) \frac{\sigma(d) ^2}{2}\right)$$

Let us introduce our next concept of well-{separated}ness. 
\begin{definicja} \label{def:ressep}
Let us consider a set of clusters separated as follows: 
Lat $\Gamma=\{C_1,\dots,C_k\}$ be a partition of the dataset $S$, $d$ be a pseudo-distance. 
Let   
\begin{equation}\label{eq:resdist}
d(i,l)>\sqrt{2}\sqrt{\beta(\Gamma,d) }
\end{equation}
for each $i,l$ 
such that $i$ belongs to a different cluster than $l$ under $\Gamma$.  
Then we say that the set $S$ with distance $d$ is 
\emph{ residually $k$-separable}. 
\end{definicja}

\begin{twierdzenie}
Assume 
that the set $S$ with distance $d$ is 
\emph{ residually $k$-separable}.   Then 
$\Gamma$ minimizes $Q(\Gamma,d)$ over all clusterings of the dataset $S$. 
\end{twierdzenie}

\begin{proof}
    Assume that a clustering $\Gamma'$ different from $\Gamma$ is the optimum. Then there exist at least two {data}points $P,Q$ belonging to the same cluster under $\Gamma'$, but not under $\Gamma$. In such a case $d(P,Q)\ge \sqrt{\beta(\Gamma,d)}$
$$
Q(\Gamma',d)
\ge 
\sum_{j=1}^k \frac{1}{2n_j} \sum_{\mathbf{x}_i \in C_j} 
\sum_{\mathbf{x}_l \in C_j} \sigma(d) ^2  
+
\frac{(2\left(Q(\Gamma,d) 
- (n-k-1) \frac{\sigma(d) ^2}{2} \right)-\sigma(d)^2}{2} 
$$

$$ 
= (n-k) \frac{\sigma(d) ^2}{2} 
+Q(\Gamma,d) 
- (n-k-1) \frac{\sigma(d) ^2}{2} 
- \frac{\sigma(d) ^2}{2}
= Q(\Gamma)+(n-k) \frac{\sigma(d) ^2}{2} 
$$

\end{proof}

\begin{twierdzenie}
Given  two pseudo-distance functions $d_1,d_2$ over $S$ such that for any two distinct $x,y$: $d_2^2(x,y)=d_1^2(x,y)+\Delta$ for some constant $\Delta$. Then 
$$\beta(\Gamma,d_2)= 
=
\beta(\Gamma,d_1)  
+ \Delta
$$     
\end{twierdzenie}
\begin{proof}
    
$$ Q(\Gamma,d_2)   
=\sum_{j=1}^k \frac{1}{2n_j} \sum_{ i \in C_j} 
\sum_{ l \in C_j} d_2(i,l) ^2  
=\sum_{j=1}^k \frac{1}{2n_j} \sum_{ i \in C_j} 
\sum_{ l \in C_j} \left(d_1(i,l) ^2  +\Delta\right)
$$ $$ 
=Q(\Gamma,d_1) +\sum_{j=1}^k \frac{1}{2n_j} \sum_{ i \in C_j} 
\sum_{ l \in C_j}  \Delta 
=Q(\Gamma,d_1) +(n-k) \frac{\Delta}{2} 
$$   
Furthermore
$$\beta(\Gamma,d_2)=
2\left(Q(\Gamma,d_2)-(n-k-1) \frac{\sigma(d_2) ^2}{2}\right) 
$$ 
$$= 
2Q(\Gamma,d_2)
-(n-k-1)  \sigma(d_2) ^2
=
2Q(\Gamma,d_1) +(n-k) \Delta
-(n-k-1)  \sigma(d_1) ^2
-(n-k-1) \Delta
$$ 
$$= 
2Q(\Gamma,d_1)  
-(n-k-1)  \sigma(d_1) ^2
+ \Delta
=
\beta(\Gamma,d_1)  
+ \Delta
$$ .
\end{proof}

The above theorem implies:
\begin{twierdzenie} \label{th:addingDelta}
Given  two pseudo-distance functions $d_1,d_2$ over $S$ such that for any two distinct $x,y$: $d_2^2(x,y)=d_1^2(x,y)+\Delta$ for some constant $\Delta$. 
Then  
if   the set $S$ with pseudo-distance $d_1$ is 
\emph{ residually $k$-separable} then   the set $S$ with pseudo-distance $d_2$ is 
\emph{ residually $k$-separable}. 
\end{twierdzenie}
\begin{proof}
   As $\beta(\Gamma,d_2)
=
\beta(\Gamma,d_1)  
+ \Delta
$, then it would be sufficient for \emph{ residual $k$-separability} of $S$ under $d_2$ that the squared pseudo-distance between elements of distinct clusters is increased by $\Delta$ which is the case by definition of $d_2$. 
So increase of distances from $d_1$ to $d_2$ preserves the residual $k$-separation.
\end{proof}

\begin{definicja}
We say that a clustering function $f(d)$ returns
\emph{ residual $k$-clustering} of $S$ 
if $S$ is resudually  $k$-separable under $d$ and $f(d)$ returns the $\Gamma$ clustering used in the definition of  
\emph{ residual $k$-separability } - Def. \ref{def:ressep} . 
\end{definicja}

\begin{definicja}
We say that a clustering function $f(d)$ returns
\emph{ residual range-$k_x$ clustering} of $S$ 
if $S$  residually $k$-separable under $d$ for some $1\le k\le k_x $ and $f(d)$ returns the $\Gamma$ clustering used in the definition of  
\emph{ residual $k$-separability } - Def. \ref{def:ressep} 
and
for no cluster $C\in \Gamma$
there exists  $k'$, $2\le k`\le k_x-k+1$ that  $C$ is 
residually $k'$-separable. 
The maximal $k$ with this property shall be called the level of residual range-$k_x$ clustering.
\end{definicja}

Obviously, $k$-means++ is no more suitable sub-algorithm for this task of the following master algorithm: try out all k=2 to $k_x$ if there exist residual $k$-clustering, and if so, then check each sub-cluster on no residual $k'$ separability. 

We need to create the following modification of $k$-means++: (res-$k$-means++). Instead of taking squared distances to the closest seed, use the difference between it and the squared smallest distance whatsoever during initialization stage. 

Let us ask the question how difficult it would be to discover the optimal clustering. 
Let us consider the res-$k$-means$++$ algorithm, or more precisely the derivation of the initial clustering. Let us consider a step when $i$ seeds have hit $i$ distinct clusters. 
Then the probability of hitting an unhit cluster in the next step amounts to:
$$\frac{SSD_{unhit}}{SSD_{unhit}+SSD_{hit}}=1- \frac{SSD_{hit}}{SSD_{unhit}+SSD_{hit}}$$
where $SSD_{unhit}$ is the sum of squared distances to closest seed minus $\delta(d)^2$ from elements of unhit clusters, and   $SSD_{hit}$ is the sum of squared distances  minus $\delta(d)^2$ to closest seed from elements of hit clusters. Obviously 
$$E(SSD_{hit})= 2 \beta(\Gamma,d)$$
as the expected squared distance from cluster center amounts to $Q(\{C\},d$). On the other hand
$$SSD_{unhit} \ge 2 \beta(\Gamma,d) \sum_{j=i+1}^k n_j $$
If we assume that  the cardinality of all clusters is the same and equals $m$, then we have 
$$ 
\frac{SSD_{unhit}}{SSD_{unhit}+SSD_{hit}} \ge
1-\frac{2 \beta(\Gamma,d)}{2 \beta(\Gamma,d)*(k-i)m+2 \beta(\Gamma,d)}
=1- \frac{1}{m(k-i)+1}
$$
So that the overall expected probability of hitting all clusters during initialization amounts to at least (as in eq. (\ref{eq:hitprob}))
$$ \prod_{i=1}^{k-1} \left(1- \frac{1}{m(k-i)+1} \right)
$$
Remarks on high probability and unequal cluster sizes are here the same as with  eq. (\ref{eq:hitprob}).

Now consider the essential of our considerations: what will happen when performing Kleinberg's consistency operation. 

The first problem that we encounter is that consistency transform performed on a cluster that is not {residual}ly $k'$ separable, may turn to {residual}ly $k'$ separable one. 
Therefore we need to restrict consistency transformation to a discrete convergent one. 
But this is not sufficient. As we make use now of the concept of the smallest distance in our formulas, we need to add the restriction that the lowest distance will not be decreased. 

\begin{twierdzenie}
    The residual range-$k_x$ clustering at the level $k$ will remain the residual range-$k_x$ clustering at the level $k$ after discrete convergent consistency transform keeping lowest distance. In other words  discrete convergent consistency transform keeping the lowest distance preserves clustering by a function detecting residual range-$k_x$ clustering. 
\end{twierdzenie}
\begin{proof}
The increase of inter-cluster distances does not violate residual $k$-separation. 
The decrease of intra-cluster separation according to the imposed limitations does not turn a non-{residual}ly separable set into a separable set. 
\end{proof}

This implies 
\begin{twierdzenie}
     A function detecting residual range-$k_x$ clustering has the property of scale-invariance, discrete convergent consistency keeping lowest distance and range-$k_x$ richness. 
\end{twierdzenie}

What is more,   consider the  discrete convergent consistency transformation keeping lowest distance  of a distance $d$ to a distance $d'$.  
As shown in \cite{RAKMAKSTW:2020:trick}, a distance function $d'$ being non-euclidean can be turned into Euclidean one $d"$ by adding an appropriate constant $\delta^2$ to each squared distance $d'(i,j)^2$ and the clustering with (kernel) $k$-means will preserve the $k$-clustering. And, with the above Theorem \ref{th:addingDelta},  the property of residual $k$ separability will be preserved.  

\section{Discussion} \label{sec:discussion} 

Kleinberg created a set of ''natural'' axioms for distance-based clustering functions which turned out to be contradictory. This axiomatic system was cited hundreds of time in the literature without pointing at the fact that this is a sick situation to have a well-established domain of clustering without having clarified the basic concepts like clusters, or the counterproductive outcome of this axiomatic set being the claim that the most popular clustering algorithm, the $k$-means algorithm, is not a clustering algorithm at all.  
The natural suggestion that arises from this situation is to pose the question whether or not an algorithm needs to produce an axiomatizable outcome in case that there  is no cluster structure in the data. Rather we should assume  the position ''garbage in - garbage out'' and concentrate to define properties for clustering algorithms when applied to  {data}sets  where there are real clusters.  

This situation led to a branch of research trying to answer the fundamental question whether or not there are clusters in the dataset. 
This note can be treated as a contribution in this direction. 
We define data structures with apriori known optima for $k$-means algorithm. Such structures can be used as an extreme case testbed for algorithms in this clustering family. 

But they shed also some light on the clustering axioms of Kleinberg. 
We showed that it is possible to define {data}sets with clean structures for various numbers of clusters, achieving partial compliance with the richness axiom of Kleinberg. Purely theoretically, our approximation via "range" richness could be extended to "range n" richness. But we deliberately imposed restrictions that a cluster should have at least two elements  and also insisted on $km$ so that the classical $k$-means++ algorithm can discover clustering with high probability. If the range should be extended behind $m$, then the distances to smaller clusters would need to be increased.  In all, this is possible.  

The title of this note contains the phrase "High-Dimensional". This is due to the fact that the operation of embedding data after consistency transformation may frequently lead to a high dimensional embedding, not present in the original data. 

We insist on using "Wide Gap" when analysing $k$-means clustering from theoretical point of view because the goal of $k$-means is not to find well separated data but rather data that minimize the sum of distance squares to cluster centers. It may lead to splitting larger "optical" clusters and merge with them small "optical" clusters if the small clusters are not far enough from the large clusters.

\section{Conclusions} \label{sec:conclusions}

This note presented two new structures of {data}sets to be clustered via $k$-means with known apriori optimal clustering. 
Both, when subjected to Kleinberg's convergent consistency transformation keeping the shortest distance, preserve the respective structural properties. But the consistency transformation leads outside the realm of Euclidean spaces. When restoring Euclidean distances, only one of them maintains its structural properties so that Kleinberg's axioms are applicable to it. 

Future research should point at directions of changes in the Kleinberg's axioms in cases when the gaps between clusters are smaller than those proposed here. 

\bibliographystyle{plain}
\bibliography{The_bib,Meine_bib}

\end{document}